\documentclass[letterpaper]{article} 
\usepackage[]{aaai25}  
\usepackage{times}  
\usepackage{helvet}  
\usepackage{courier}  
\usepackage[hyphens]{url}  
\usepackage{graphicx} 
\usepackage{natbib}  
\usepackage{caption} 
\usepackage{algorithm}
\usepackage{algorithmic}

\usepackage{newfloat}
\usepackage{listings}
\DeclareCaptionStyle{ruled}{labelfont=normalfont,labelsep=colon,strut=off} 
\floatstyle{ruled}
\newfloat{listing}{tb}{lst}{}
\floatname{listing}{Listing}
\nocopyright 

\usepackage{amsthm}

\newtheorem{definition}{Definition}
\newtheorem{assumption}{Assumption}
\newtheorem{theorem}{Theorem}
\newtheorem{lemma}{Lemma}
\newtheorem{remark}{Remark}

\usepackage{amsmath,amsfonts,bm}
\usepackage{tikz}

\usepackage{booktabs}

\title{Multi-Robot Task Allocation using Global Games with Negative Feedback: The Colony Maintenance Problem 
}
\author {
    Logan E. Beaver
}
\affiliations{
    Department of Mechanical and Aerospace Engineering\\Old Dominion University,
    Norfolk, VA USA 23529\\
    lbeaver@odu.edu
}

\begin{document}

\maketitle

\begin{abstract}
In this article we address the multi-robot task allocation problem, where robots must cooperatively assign themselves to accomplish a set of tasks.
We consider the colony maintenance problem as an example, where a team of robots are tasked with continuously maintaining the energy supply of a central colony.
We model this as a global game, where each robot measures the energy level of the colony, and the current number of assigned robots, to determine whether or not to forage for energy sources.
The key to our approach is introducing a negative feedback term into the robots' utility, which also eliminates the trivial solution where foraging or not foraging are strictly dominant strategies.
We compare our approach qualitatively to existing an global games approach, where a positive positive feedback term admits threshold-based decision making that encourages many robots to forage.
We discuss how positive feedback can lead to a cascading failure when robots are removed from the system, and we demonstrate the resilience of our approach in simulation.
\end{abstract}

\section{Introduction}

As we continue deploying robots in real outdoor environments, there is a growing interest in using robots for remote sensing \cite{naderi2022sharing}, environmental data collection \cite{dunbabin2012robots}, and ecological monitoring \cite{notomista2022multi}.
These applications require a robust, resilient, and flexible approach to long-duration autonomy, and the there has been a significant research effort toward ecologically-inspired robotics \cite{Egerstedt2018RobotAutonomy,beaver2021overview} to deal with the messy uncertainties of real-world deployment.

The focus of this work is to explore a global games-inspired strategy for the colony maintenance problem, where robots must periodically perform a maintenance tasks, e.g., inspection, charging, and energy harvesting, to maintain an autonomous colony in a remote location.
This work also considers the presence of a human agent, who may recruit a subset of the colony's robots as part of an external non-maintenance task.
This requires the system to efficiently allocate robots to tasks in a way that is robust to the removal of individuals and can quickly adapt to changes in the uncertain environment.

This work falls broadly in the domain of Multi-Robot Task Allocation (MRTA).
MRTA is one of the most challenging open problems in mutli-robot systems research \cite{khamis2015multi}; despite the great number of MRTA algorithms, critical components, such as dynamic task allocation and the use of heterogeneous robots, present a significant open problem in the literature.
This work considers a single task robot + multiple robots per task (ST-MR) formulation \cite{chakraa2023optimization} of MRTA, where the robots are dynamically allocated to a task that has its own time-varying dynamics.

In general, MRTA requires solving a combinatorial optimization problem to assign robots to tasks, and this does not lend itself to the kind of reactive, robust, and resilient formulation required for dynamic allocation.
This is apparent in a recent MRTA survey article \cite{chakraa2023optimization}, which estimates that at least $75\%$ of the MRTA literature requires either 1) multiple rounds of coordination between robots or 2) offline computation of a static strategy.
This requires either a centralized authority to solve a complex mixed-integer optimization problem, or it necessitates multiple rounds of communication until all robots reach a consensus on the task assignment.
As a consequence, any change in the system (e.g., the addition or removal of robots) requires the assignment to be re-computed.

While there have been advances in robot allocation for dynamic systems \cite{notomista2019optimal}, game-theoretic approaches make up only about $4\%$ of published solutions \cite{chakraa2023optimization}.
All but one of these approaches use game theory to generate coalitions \cite{MARTIN2023104314}, which are then assigned to tasks using traditional integer optimization-based methods.
In fact, only \cite{Kanakia2016ModelingGame} uses global games for multi-robot task allocation.

In a global game, robots in the system measure a global \textit{signal} (or stimulus) corresponding to a task, e.g., the amount of energy available in a battery array, or the intensity of a fire in a fire-fighting task \cite{Kanakia2016ModelingGame}.
It has been proven that, under reasonable conditions, a threshold-based strategy exists and is a Bayesian Nash Equilibrium 
\cite{Kanakia2016ModelingGame,Mahdavifar2018GlobalSharing}.
This implies that robots simply compare their measurement of the global signal to an internal threshold, and assign themselves to the task when the threshold is exceeded.
A qualitatively similar approach was used in a biologically-inspired colony maintenance problem \cite{krieger2000call}, without the formalism of global games.

While the threshold-based solution of global games is appealing, we demonstrate in this work that the standard approach is neither resilient nor adaptive in the MRTA context.
Global games are well-suited for one-shot assignments (i.e., the assignment is decided once at the beginning) as discussed in \cite{Kanakia2016ModelingGame}.
However, one-shot games cannot adapt to changes in the global signal (stimulus), and they are not robust to the addition or removal of robots during task completion.
To summarize, the contributions of this article are as follows:
(i) a global game-based framework for the assignment of robots to a dynamic task,
(ii) an analysis of how negative feedback in the robot's objective makes the system resilient to the removal of robots by a human, and 
(iii) a discussion of how positive feedback in the global game can lead to the trivial equilibrium where either all or none of the robots forage.
We note that these are particularly useful for long-duration autonomy tasks, where the robots ought to expend as little energy as possible while handling tasks as they arise in the environment.

The remainder of the article is organized as follows.
We first present the colony maintenance global game and list our working assumptions.
Then, we analyze the Nash equilibrium of the global game, and we discuss the implications of feedback on the resilience of the system.
Then, we present simulation results to demonstrate the performance and resilience of our solution, before finally discussing conclusions and future work.

\section{Problem Formulation} \label{sec:problem}

We consider the problem of maintaining an autonomous colony using a team of $N$ robots under a global games framework.
We model each robot as a single integrator operating in a planar environment,
\begin{equation} \label{eq:robot-dynamics}
    \dot{\bm{p}_i} = \bm{v}_i,
\end{equation}
where $i$ indexes the robot and $\bm{p}_i, \bm{v}_i \in\mathbb{R}^2$ are the state (position) and control input (velocity), respectively.

Our goal is to assign the robots to a number of tasks in the environment.
These tasks can be state and time-varying, and their dynamics are unknown a priori.
Inspired by \cite{krieger2000call}, 
we consider an examplar task to maintain
the energy level of a centralized \textit{colony} by foraging for \textit{energy sources}.
The robots achieve this by randomly searching the environment, where energy sources may include fuel, combustibles, or organic material.
We describe the energy level of the colony with a real-valued \textit{signal},
\begin{equation}
    s(t) \in [0, 1].
\end{equation}
The signal $s(t)$ is a time-varying function that measures the percentage of energy left in the system,
for example, by measuring the voltage of a battery array or the mass of fuel in a tank.
The signal has possibly unknown dynamics, and its exact implementation is application dependent.
For example, a package delivery task may use the number of packages as a signal, whereas a maintenance task may have the signal increase with time between inspections.

Our goal is to allocate robots to the foraging task to ensure that the energy level remains above zero, i.e.,
\begin{equation}
    s(t) \geq 0.
\end{equation}
We also consider the presence of a human colonist, who can recruit a number of the foraging robots for an external task.
The human colonist may recruit any robots that are idle, actively foraging, or some combination of the two.

In a traditional robot-task allocation problem, it may be practical to have the colony act as a centralized controller that uses the signal $s(t)$ to determine which of the $N$ robots should forage for energy sources at each time-instant.
However, this is effectively a mixed-integer program \cite{chakraa2023optimization}, which can quickly grow intractable for systems with many robots.
This complexity is compounded by the presence of a human, which will require the assignment to be recomputed whenever foraging robots are removed.
Instead, we propose a \textit{global game}, where each robot measures the global signal $s(t)$ and the number of foraging robots (both known by the colony) in order to determine whether it should forage for energy sources.
While this requires some centralized information about the signal level, our approach explicitly avoids the large integer program associated with traditional task assignment.
We also emphasize that, with an eye toward more general MRTA problems, we explicitly avoid the trivial solution where a single threshold on $s(t)$ determines whether all robots or no robots forage.

We approach this problem from the perspective of \textit{mechanism design}, i.e., we design the utility functions of the robots to ensure the Nash equilibrium of the induced game has the desirable properties discussed above.
First, we start with the description of a strategic form game \cite{chremos2020game}.
\begin{definition} \label{def:game}
    A finite normal-form game is a tuple 
    \begin{equation*}
        \mathcal{G} = \big(\mathcal{I}, \mathcal{S}, u_i \big),
    \end{equation*}
    where
    \begin{itemize}
        \item $\mathcal{I} = \{1, 2, \dots, N\}$ is the finite set of players,
        \item $\mathcal{S}$ is the set of strategies, and
        \item $u_i,\, i\in\mathcal{I}$ is the utility function for each player.
    \end{itemize}
\end{definition}
Note that the set of strategies in Definition \ref{def:game} is a collection of binary variables, where $1$ corresponds to foraging and $0$ corresponds to idling at the colony.
To make Definition \ref{def:game} consistent with the global games literature \cite{Morris2001GlobalApplications}, we employ the following change of coordinates for the signal $s(t)$,
\begin{equation}
    \theta(t) := 1 - s(t).
\end{equation}
Intuitively, $\theta(t)$ corresponds to the ``urgentness'' of the foraging task, and we seek to ensure $\theta(t) \leq 1$.

\begin{remark} \label{rmk:utility}
    In general, we require the utility function
    \begin{equation*}
        u_i(a, n, \theta) \to \mathbb{R}
    \end{equation*}
    to be,
    (i)  strictly increasing in $\theta$ and
    (ii) strictly decreasing in $n$.
    These properties are sufficient to generalize our results that follow.
\end{remark}

As an illustrative example, we define a utility function for each robot, which we use to control the system behavior,
\begin{equation} \label{eq:utility}
    u_i(a, n, \theta) =
    -(n+a) c_a + a\Big(\kappa + \lambda e^{-(n+1)} + \theta\Big),
\end{equation}
where $a \in\{0, 1\}$ is the decision of the robot, $n$ is the number of other robots assigned to the task, $c_a$ is the cost of assigning a robot to forage, and $\kappa, \lambda$ are constant coefficients that control the shape of the utility function.
This is based on the example presented in \cite{Kanakia2016ModelingGame}, with minor modifications to satisfy Remark \ref{rmk:utility}.

The intuition behind the utility function \eqref{eq:utility} is as follows: the first term penalizes each robot by a fixed cost $c_a$; this represents the nominal energy required for foraging.
The second term is the incentive for individual robots to forage.
There is a constant reward $\kappa$, the signal $\theta(t)$, and an exponential term that decreases with the number of other robots that are foraging.
The first two constants act as a reward for foraging, and the exponential is a negative feedback term that stops foraging from being a strictly dominant strategy.
To make our mechanism consistent with our signal $\theta\in[0, 1]$, we impose the following working assumptions.

\begin{assumption} \label{smp:kappa}
    The scalar $\kappa$ satisfies $c_a - \kappa_a \leq 1$.
\end{assumption}

\begin{assumption} \label{smp:lambda}
    The scalar $\lambda$ satisfies $0 < \lambda \leq e(c_a - \kappa)$.
\end{assumption}

These are technical assumptions about the game to ensure there isn't a trivial solution, i.e., a mixed Nash equilibrium must exist.
These assumptions are specific to the mechanism \eqref{eq:utility} and the domain of our signal $s(t)$, but they are straightforward to generalize when the utility $u$ is monotonic in $n$ and $\theta$ by following our analysis.

\section{Solution Approach} \label{sec:solution}

We follow the notation of \cite{Morris2001GlobalApplications} for our solution approach.
First, we define the marginal utility for each robot $i$,
\begin{align}
    \pi_i(n, \theta) &= u_i(1, n, \theta) - u_i(0, n, \theta) \notag\\
    ~&= -c_a + \kappa + \lambda e^{-(n+1)} + \theta, \label{eq:marginal-utility}
\end{align}
which captures the benefit of switching from $a = 0$ (idle) to $a = 1$ (forage).
As all robots are homogeneous, we drop the subscript $i$ for the remainder of the article.

The fact that $\pi(n, \theta)$ is decreasing in $n$ is a significant departure from the existing approaches \cite{Morris2001GlobalApplications,Kanakia2016ModelingGame,Mahdavifar2018GlobalSharing}.
Having $\pi$ be non-decreasing with respect to $n$ is a fundamental feature of global games in the economics literature \cite{Morris2001GlobalApplications}.
This creates \textit{positive feedback}; any robot taking the action $a = 1$ increases $n$, which further increases the marginal utility $\pi$ for all robots.
The increase in $\pi$ leads to more robots taking the action $a=1$, until all robots are foraging in the environment.
This is particularly useful, for example, when modeling a bank run or debt crisis \cite{Morris2001GlobalApplications}.

In contrast, we propose that $\pi(n, \theta)$ should decrease in $n$, which creates \textit{negative feedback}.
Any robot taking the foraging action $a=1$ increases $n$, which decreases $\pi$ and makes foraging unattractive to other robots.
As we discuss at the end of this section, this negative feedback makes our approach resilient to the removal of robots by a human.
Despite this key difference, our approach shares many traits with standard global games, and we present these in the following results.
We start with two standard definitions, specialized to our problem.

\begin{definition}[Strictly Dominant Strategy] \label{def:dom-strat}
The action $a^*$ is a \textit{strictly dominant strategy} for a given $\theta$ if and only if
\begin{equation}
    u(a^*, n, \theta) > u(a', n, \theta),
\end{equation}
for  all $n$ and $a' \neq a^*$. 
\end{definition}

\begin{definition}
[Mixed Strategy] \label{def:mix-strat}
A mixed strategy for a given $\theta$ is a probability distribution over a support $\delta\subseteq\mathcal{S}$ such that,
\begin{equation}
    \mathbb{E}\Big[u(a, n, \theta) \Big] = \mathbb{E}\Big[u(a', n, \theta) \Big],
\end{equation}
for all $a, a' \in\delta$.
\end{definition}

\begin{lemma} \label{lma:theta-bnds}
    The induced game has a strictly dominant strategy equilibrium for values of $\theta$ satisfying,
    \begin{equation}
        \theta \in \big(-\infty, c_a - \kappa - \lambda e^{-1}\big] \bigcup \big[c_a - \kappa, \infty\big).
    \end{equation}
\end{lemma}

\begin{proof}
    First we bound $\pi(n, \theta)$ by picking the extreme values of $n\in[0, \infty)$,
    \begin{equation}
        -c_a + \kappa + \theta  \leq \pi(n, \theta) \leq -c_a + \kappa + \lambda e^{-1} + \theta.
    \end{equation}
    When the lower bound of $\pi$ is positive, this implies that the optimal solution is to play $a=1$ (forage) regardless of the behavior of other robots.
    Similarly, when the upper bound of $\pi$ is negative, the optimal strategy is to play $a=0$ (do not forage) regardless of the behavior of the other robots.
    These are strictly dominant strategies by definition.
\end{proof}

\begin{theorem}\label{thm:NE}
    For every value of 
    \begin{equation}
        \theta\in(c_a-\kappa-\lambda e^{-1}, c_a - \kappa),
    \end{equation}
    the  Nash equilibrium is achieved with a mixed strategy that satisfies
    \begin{equation}
        \pi(n, \theta) = 0.
    \end{equation}
\end{theorem}

\begin{proof}
    By construction, the marginal utility satisfies
    \begin{align*}
        \pi(0, \theta) < 0 \\
        \pi(\infty, \theta) > 0
    \end{align*}
    for all values of $\theta$ under the premise of Theorem \ref{thm:NE}.    
    By continuity of $\pi$ in $n$, there is some $n^*$ such that
    \begin{equation*}
        \pi(n^*, \theta) = 0.
    \end{equation*}
    The strict monotonicity of $u$ in $n$ implies that $n^*$ is unique, and thus $n^*$ corresponds to a unique Nash equilibrium.
    By definition, the equilibrium cannot be a strictly dominant strategy, and thus it must be a mixed strategy.
\end{proof}

We note that, while it appears counter-intuitive to have $n^*$ be a non-integer number of robots, this is a natural consequence of following a mixed-strategy.
As we discuss in the following section, $n^*$ is the \textit{expected value} for the number of robots to forage, which is a non-integer number in general.

\begin{remark} \label{rmk:limits}
    When the inequalities of Assumptions \ref{smp:kappa} and \ref{smp:lambda} are strict equalities, the game has mixed-strategy Nash equilibrium for values of $\theta$ satisfying,
    \begin{equation}
        \theta \in [0, 1].
    \end{equation}
\end{remark}

To improve the robustness of the system, the bounds of Remark \ref{rmk:limits} can be tightened by considering strict inequalities.
This implies a smaller set $\theta\in [\underline{\theta}, \overline{\theta}] \subset [0, 1]$, outside of which the strictly dominant Nash equilibrium exist.

The shape of our marginal utility function \eqref{eq:marginal-utility} is presented in Fig. \ref{fig:marginal-utility} for our proposed mechanism.
The payoff $\pi(n, \theta)$ is plotted against the signal $\theta$ for the case described by Remark \ref{rmk:limits}.
The top and bottom dashed lines correspond to the minimum and maximum value of $\pi(n, \theta)$ as described by Lemma \ref{lma:theta-bnds}, and the entire shaded region corresponds to values of $\theta$ where a mixed strategy Nash equilibrium exists.
Our proposed mechanism's marginal utility, depicted in Fig. \ref{fig:marginal-utility}, has a number of interesting implications for the behavior of the system.

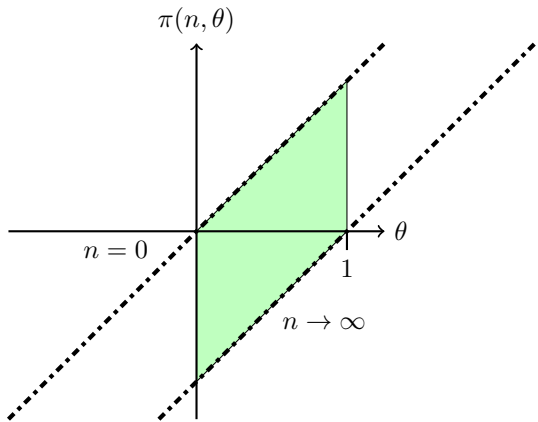
\begin{figure}[ht]
    \centering
    \begin{tikzpicture}
        \draw[fill=green!25] (0, 0) -- (0, -2) -- (2, 0) -- (2,2) -- cycle;
        \draw[thick,->] (-2.5, 0) -- (2.5, 0) node[right] {$\theta$};
        \draw[thick,->] (0, -2.5) -- (0, 2.5) node[above] {$\pi(n, \theta)$};
        \draw[ultra thick, dash dot] (-2.5, -2.5) -- (-0.5,-0.5) node[above left] {$n=0$} -- (2.5, 2.5);
        \draw[ultra thick, dash dot] (-2.5+2, -2.5) -- (1, -1) node[below right]{$n\to\infty$} -- (2.5+2, 2.5);
       \draw[thick] (2, 0) -- (2,-0.25) node[below]{$1$};
    \end{tikzpicture}
    \caption{The marginal utility of our proposed mechanism. Robots do not consider the actions of others when $\theta$ falls beyond the shaded green area.}
    \label{fig:marginal-utility}
\end{figure}

\begin{remark} \label{rmk:no-feedback}
    If the utility $u$ is independent of $n$, i.e., there is no feedback with respect to the number of robots, then the system has a unique $\theta^*$ where the system switches between the two dominant strategies.
\end{remark}

Remark \ref{rmk:no-feedback} is clear by substituting $\lambda = 0$ into the bounds of Lemma \ref{lma:theta-bnds}.
In this case the upper and lower bounds of $\pi(n, \theta)$ converge, and the parallel lines in Fig. \ref{fig:marginal-utility} become co-linear.
While this may seem like a trivial result, it demonstrates that having feedback with respect to the number of active robots in the marginal utility function \eqref{eq:marginal-utility} plays a critical role in regulating the behavior of the system.

\begin{remark} \label{rmk:cost}
    If $\theta(t)$ increases continuously, then making the cost $c_a$ heterogeneous between robots while fixing $\lambda$ and $\kappa$ implies that more capable robots will always begin foraging before less capable robots.
\end{remark}

The proof of Remark \ref{rmk:cost} comes from the fact that increasing the cost $c_a$ directly increases $\pi(n, \theta)$ by the same amount.
Let $\pi_1 < \pi_2$, where $\pi_1$ is a robot with a lower cost $c_a$ (i.e., it is more capable) and $\pi_1(0, 0) = 0$.
Then, there always exists some value of $\theta\in(0, 1)$ such that
\begin{equation}
    \pi_1(n, \theta) > 0 > \pi_2(n, \theta).
\end{equation}
Thus, as $\theta$ increases continuously, the marginal utility of $\pi_1$ will become positive before $\pi_2$, and the more capable robot will be assigned to the foraging task.

Remark \ref{rmk:no-feedback} emphasizes the importance of a feedback term where the number of foraging robots affects the marginal utility of all other robots.
Furthermore, Remark \ref{rmk:cost} demonstrates the our proposed approach has expected and intuitive behavior for systems with heterogeneous robots.
Finally, we present two results that describe the \textit{resilience} of our system with respect to the addition or removal of robots by a human.

\begin{lemma} \label{lma:recruit-any}
    When a human recruits $k < N$ robots for an external task, the upper threshold $\theta \leq \overline{\theta}$, where foraging becomes a strictly dominant strategy, decreases.
\end{lemma}

\begin{proof}
The lower bound on $\pi(n, \theta)$ is approximated for a large swarm of robots where $n\to\infty$.
For any finite $N$, the lower bound on $\pi(n, \theta)$ is increased and brought closer to the upper bound.
This shifts the bottom dashed line of Fig. \ref{fig:marginal-utility} up, which moves its $y$-intercept to the left.
This corresponds to a decrease in the upper threshold $\overline{\theta}$.
\end{proof}

\begin{lemma} \label{lma:recruit-forage}
    When a human recruits a robot that is foraging, the marginal utility $\pi(n, \theta)$ increases for all robots.
\end{lemma}

\begin{proof}
Lemma \ref{lma:recruit-forage} follows from $\pi(n, \theta)$ decreasing in $n$.
\end{proof}

These results demonstrate the resilience of our global games formulation.
Lemma \ref{lma:recruit-any} implies that by removing some robots from the system, the threshold where the system tips to an ``all robots forage'' dominant strategy equilibrium decreases.
In other words, as the energy level approaches zero, more robots will begin to forage to compensate for those removed by the human.

Similarly, Lemma \ref{lma:recruit-any} implies that robots from the colony have an incentive to replace any foraging robots that are recruited by the human.
This points to an intuitive and resilient system behavior: send a higher fraction of robots to forage as the number of robots decreases.
This is in contrast to approaches using traditional global games, as we discuss next.

\begin{remark} \label{rmk:human-game}
    For a traditional global game where $\pi$ increases with $n$, Lemmas \ref{lma:recruit-any} and \ref{lma:recruit-forage} are inverted, i.e., recruiting robots raises the lower bound $\underline{\theta} \leq \theta$, and recruiting foraging robots reduces $\pi(n, \theta)$.
\end{remark}

Remark \ref{rmk:human-game} implies that using traditional global games for MRTA is particularly challenging.
First, as the number of robots is decreased, it becomes less likely that robots will forage while the energy levels are high.
Critically, if a human recruits too many robots that are foraging, the marginal utility $\pi(n, \theta)$ will decrease.
This could lead to a situation where, if too many foraging robots are removed, all remaining robots remain idle in the colony--leading to a catastrophic cascading failure of the entire system!

Finally, we derive the mixed strategy that determines whether an individual robot should forage or remain idle in the colony.
To ensure the robots have sufficient time to forage (i.e., robots don't rapidly switch between foraging and idling), we implement \textit{action hysteresis}, where a foraging robot will continue foraging until it finds and returns an energy source to the next.
Meanwhile, all idle robots compute their optimal action via \eqref{eq:marginal-utility} using the current values of $s(t)$ and $n$.
The behavior of the robots is summarized as a switching system in Fig. \ref{fig:behavior}.

\begin{figure}[ht]
    \centering
    \begin{tikzpicture}[
    state/.style={draw, circle, minimum width=2cm, ultra thick,fill=red!10},
    ]
        \node[state] at (-2, 0) {Idle};
        \node[state] at (2, 0) {Forage};
        \draw[->, ultra thick] 
        (-2, 1) -- 
        node[above] {$\pi\big(n, \theta(t)\big) > 0$} 
        (2,1);
        \draw[->, ultra thick] 
        (2, -1) -- 
        node[below] {return with energy source} 
        (-2,-1);
    \end{tikzpicture}
    \caption{Switching system that describes the behavior of the robots. Each robot is idle until the marginal utility of foraging is positive, and they continue to forage until they return with an energy source.}
    \label{fig:behavior}
\end{figure}
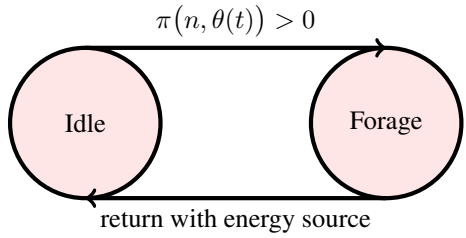

Let $n^*$ and $\theta^*$ be the current number of foraging robots and energy signal value, respectively.
If $\pi(n^*, \theta^*) \leq 0$ because $n^*$ is sufficiently large, then all idle robots remain idle.
Otherwise, if $\pi(n^*, \theta^*) > 0$, each idle robot randomly chooses to begin foraging with probability $p$.
The probability of any individual robot switching to the foraging behavior is a Bernoulli random variable, and thus the probability of of any $k$ robots foraging follows a Binomial distribution.
The expected number of robots $n'$ that start foraging is,
\begin{equation}
    \mathbb{E}\Big[n'\Big] = (N - n^*)p.
\end{equation}
After all robots select their action, the expected total number of foraging robots is,
\begin{equation*}
    \mathbb{E}\big[n\big] = n^* + \mathbb{E}\big[n'\big] = n^* + (N - n^*)p = N\,p + n^*(1-p).
\end{equation*}
Substituting this into the marginal utility \eqref{eq:marginal-utility} yields,
\begin{equation}
    \mathbb{E}\big[\pi(n, \theta)\big] = -c_a + \kappa + \lambda\,e^{-(N\,p + n^*(1-p)+1)} + \theta.
\end{equation}
By definition, setting $\mathbb{E}\big[\pi(n, \theta)\big] = 0$ and solving for $p$ defines the mixed Nash equilibrium.
\begin{equation}
    p = \frac{-ln\big(\frac{c_a - \kappa - \theta}{\lambda}\big) - n^* - 1}{N - n^*}.
\end{equation}
Note that this yields an expected number of robots,
\begin{equation} \label{eq:expected-n}
    \mathbb{E}\big[n\big] = -ln\Big(\frac{c_a - \kappa - \theta}{\lambda}\Big) - 1,\,
\end{equation}
which satisfies the Nash equilibrium of Theorem \ref{thm:NE}.
Note that the natural logarithm in \eqref{eq:expected-n} comes from the exponential term in \eqref{eq:utility}.
Selecting a different monotonic function of $n$ in the utility function \eqref{eq:utility}, e.g., $ln(n+1)$ or linear in $n$, influences the rate at which robots begin foraging.

\section{Simulation Results} \label{sec:sim}

\begin{figure*}[ht]
    \centering
    \includegraphics[width=0.45\linewidth]{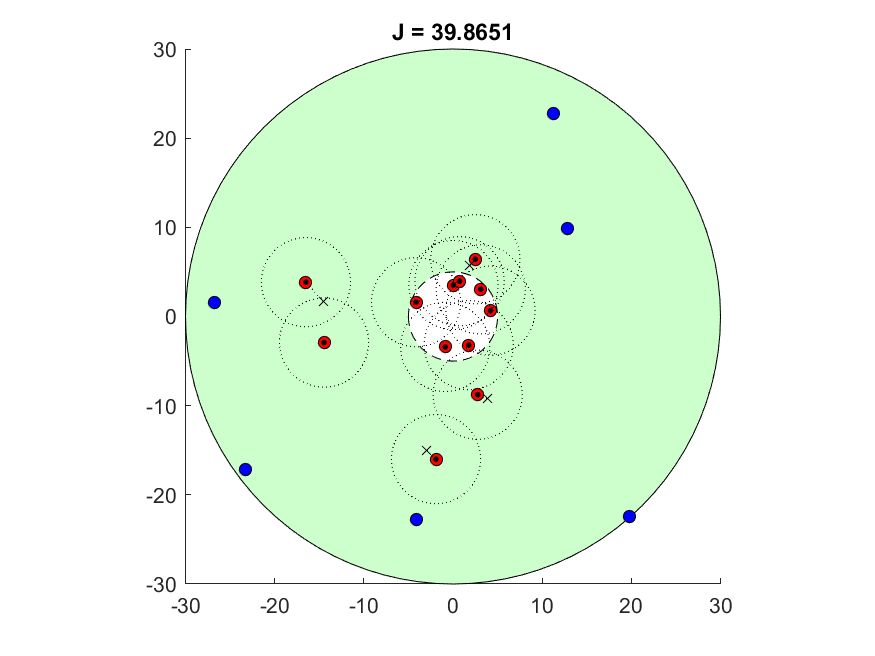}
    \includegraphics[width=0.45\linewidth]{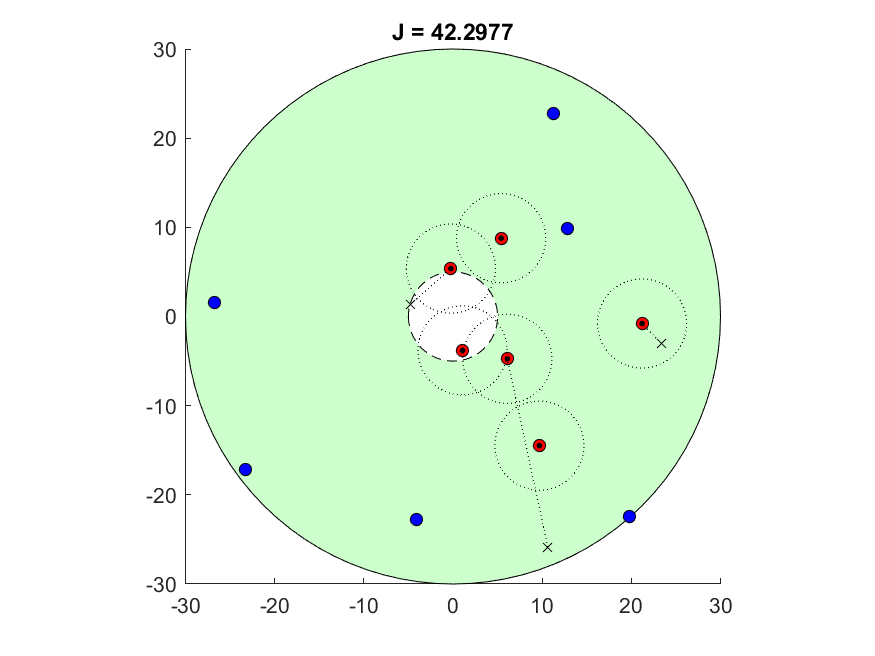}
    \caption{Snapshots of the simulation environment at $t = 800$ and $t=901$, just before and after $6$ of the $12$ robots are removed from the domain. Red and blue circles are robots, dashed circles denote the sensing radius, and straight lines point from a robot to its stored memory location.}
    \label{fig:snapshots}
\end{figure*}

We demonstrate the performance of our system in a simulation environment inspired by \cite{Krieger2000} and developed in Matlab\footnote{The simulation code is available at: \url{https://github.com/int-sys-lab/foraging-global-games}
}.
We also present a qualitative comparison between our approach and related MRTA solutions in Table \ref{tab:comparison}.
The key factors are how frequently agents are assigned to tasks (one-shot or repeated), the computational complexity of the assignment algorithm, and in what sense, if any, the resulting assignment is optimal.

\begin{table}[ht]
    \centering
    \begin{tabular}{c|cccc}
                 & Frequency & Complexity & Optimality  \\ \toprule
        Proposed & Repeated  & Linear     & Nash \\
        Kriger et al.   & Repeated  & Constant   & Heuristic \\
        Kanakia et al. & One-Shot  & Constant   & Heuristic \\
        MILP     & Repeated  & NP-Hard    & Global
    \end{tabular}
    \caption{Qualitative comparison between assignment approaches. MILP is a general mixed-integer linear program; Nash refers to the Nash equilibrium, which is globally optimal for convex problems. }
    \label{tab:comparison}
\end{table}

The colony is centered at the origin with a radius of 3 meters, and the entire domain has a radius of 30 meters.
We consider a team of $N=12$ robots, and we place $N=6$ energy sources randomly in the domain.
To avoid clumping, each energy source is placed within a $60^{\circ}$ slice of the domain at a random radius.
As with \cite{Krieger2000}, energy sources are immediately replaced after being picked up by a robot.

The colony loses energy at a rate of $0.1$ units/second through trickle losses, robots consume energy at a rate of $0.01$ units/second while moving, and each energy source provides $5$ units of energy.
This implies that it is ``cheap'' for the robots to forage, as they drain energy an order of magnitude slower than the colony itself.
We consider a total energy level of $J = 100$ units.
This ensures that the colony will entirely run out of energy in $1,000$ seconds or less, and returning an energy source only restores $5\%$ of the total energy.
Finally, we start the colony at $J=50$ units of energy to avoid an initial delay while $\theta$ is close to zero.

The robot behaviors are similar to those observed in \cite{Krieger2000}.
An idle robot $i$ attempts to remain stationary by applying $\bm{v}_i = 0$.
This ensures that it does not drain energy while waiting to forage.
Each robot that is foraging follows two behaviors.
(i) If the robot does not sense any energy sources within its sensing horizon $h > 0$, it moves to a coordinate stored in its memory.
We select this coordinate to create a random walk for the robot, which we discuss in the following paragraph;
(ii) if the robot senses an energy source within its sensing horizon, it moves to the energy source, picks it up, and returns it to the colony.

Each robot has onboard memory that is capable of storing a single coordinate in $\mathbb{R}^2$.
We use the memory of each robot to (i) return to the (noisy) last known location of energy sources, and (ii) perform a random walk.
Similar to \cite{Krieger2000}, if the robot returns an energy source to the colony and becomes idle, we clear the coordinate stored in memory.
Otherwise, if the robot continues foraging, then we store the original location of the energy source in memory and disturb it with zero-mean Gaussian noise scaled by the distance between the colony and the energy source.
This emulates the stochasticity of using odometry to return to the energy source location in \cite{Krieger2000}.
If the memory of a robot is empty, it picks a random location on the edge of its sensing radius $h$ to store in memory. 
Once the robot reaches the location in its memory, it generates another random point at the edge of its sensing horizon.
This generates a random walk through the domain with a step size of $h$ m; we use $h=5$ m for this simulation.
Once the robot identifies an energy source, it stops the random walk, picks up the energy source, and returns to the colony as described above.

Finally, to ensure safety, we use a linear control barrier function (CBF) to enforce a minimum separating distance between the robots \cite{Notomista2019Constraint-DrivenSystems} and to keep the robots within a $30$ m radius of the colony .
Thus, each robot with index $i$ solves the optimization problem,
\begin{align*}
    \min_{\bm{u}_i} &|| \bm{u}_i - \bm{u}_i^d|| \\
    \text{subject to:}& \\
    2||\bm{v}_i - v_{\max}\frac{\bm{p}_i - \bm{p}_j}{||\bm{p}_i - \bm{p}_j||}|| &\geq (R_i + R_j)^2 - ||\bm{p}_i - \bm{p}_j||^2,
\end{align*}
where $\bm{u}^d$ is the desired action as described above and $R_i, R_j$ are the collision radii of robots $i$ and $j$.
Note that this has a straightforward extension to static and dynamic obstacle avoidance.
As a consequence, the foraging robots can occasionally experience congestion while leaving the colony site.
Furthermore, the CBF may cause an idle robot to move out of the way of a foraging robot, consuming energy.

We simulate the system over a time of $1800$ seconds, and we found that the utility parameters $c_a = 1, \kappa = 0.25$, $\lambda = 5.5$ gave a reasonable profile for the expected number of robots \eqref{eq:expected-n} as a function of $\theta$.
To demonstrate the resilience of our system, half of the $N=12$ robots are recruited by a human and removed from the system at $t=900$.
The state of the system just before and after this event is presented in Fig. \ref{fig:snapshots}.
The green annulus represents the domain where energy sources are initialized, the white disk at the origin is the colony, and the red and blue circles are the robots and energy sources, respectively.
Note that at $t=800$, there are $7$ idle robots in the colony.
At $t=900$ we randomly select $6$ of the robots to remove; as described by Lemma \ref{lma:recruit-forage}, this increases the marginal benefit of foraging, and only one idle robot remains in the colony at $t=901$.

The energy of the system is depicted in Fig. \ref{fig:energy} over the $1800$ second time interval.
The red curve shows the current energy stored in the central colony; to generate the black curve, we integrate the energy cost of the foraging robots through time.
When a foraging robot returns to the colony (i.e., when the black curve moves upward), the colony energy is decreased by an amount equal to the energy consumed by the robot.
This is analogous to the robots using energy to navigate then recharging at the colony.
The blue vertical line at $t = 900$ corresponds to half the robots being removed from the domain.
This appears to enhance the performance of the system by reducing congestion.

\begin{figure}[ht]
    \centering
    \includegraphics[width=0.95\linewidth]{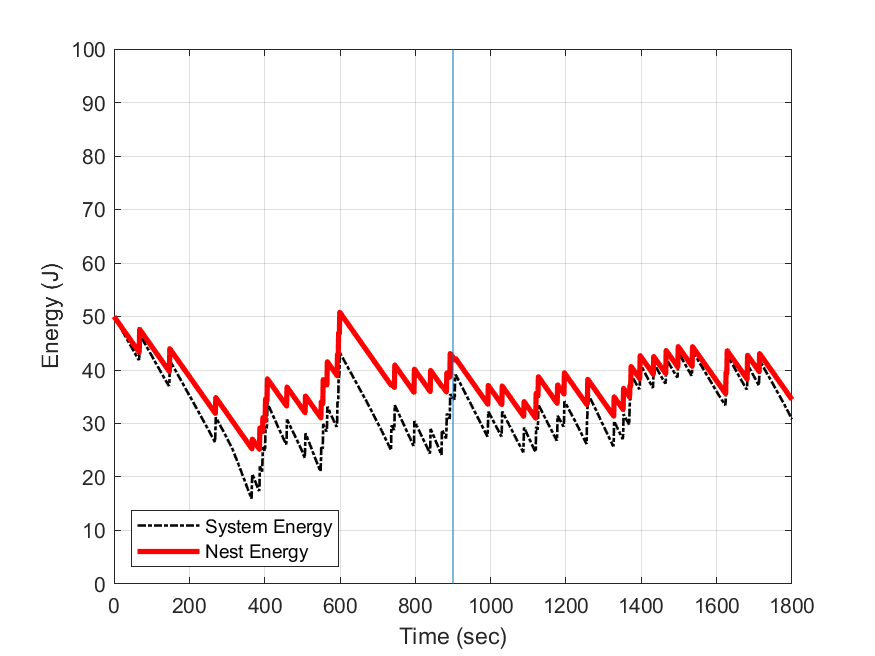}
    \caption{Energy of the colony (red) and the total energy $J$ (black, dashed). The blue vertical line is the instant where half the robots are removed from the system.}
    \label{fig:energy}
\end{figure}

Next, we present the number of robots foraging in Fig. \ref{fig:num-robots}.
The light blue shaded area represents the Nash equilibrium calculated with \eqref{eq:expected-n}, and the black curves show the number of robots foraging at each time step.
In general, the simulation results tend to have more robots foraging than the Nash equilibrium dictates.
This is because of \textit{hysteresis} in the foraging behavior (Fig. \ref{fig:behavior}), where a robot cannot stop foraging until it returns with an energy source.
Thus, if fewer robots than the Nash equilibrium begin to forage at a given time step, there is a non-zero probability that additional robots will begin to forage during subsequent time steps.
However, if more robots than the Nash equilibrium begin to forage, there is no way for them to switch to the idle behavior until they return to the colony with an energy source.
We expect that adding stochasticity, e.g., noisy measurements of the signal $\theta(t)$, will exacerbate this effect.

\begin{figure}[ht]
    \centering
    \includegraphics[width=0.95\linewidth]{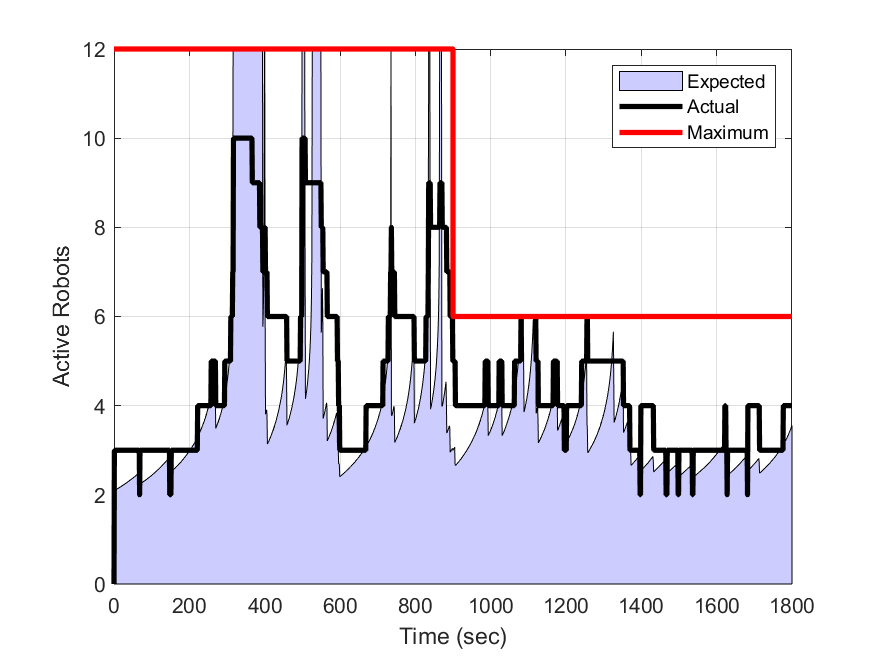}
    \caption{The expected, actual, and maximum number of robots foraging during the simulation.}
    \label{fig:num-robots}
\end{figure}

\section{Conclusions} \label{sec:conclusion}

In this work, we have explored how global games can be used to solve the multi-robot task allocation problem.
We discussed the idea of positive and negative feedback in global games, and demonstrated how negative feedback can stabilize the task assignment for a dynamic system where the game is repeatedly played.
In contrast, global games with positive feedback are more appropriate for zero-shot games, which does not provide the desirable resilience and robustness properties of our solution.
We demonstrated the capability of our approach to reach the Nash equilibrium using a mixed strategy, and we demonstrated its resilience to the removal of random robots by a human.

There are several compelling directions for future research.
Exploring the impact of robot heterogeneity on system performance is one area that is under-explored.
Including multiple global signals that the robots must maintain (e.g., foraging for energy sources and surveying an area) is an active research direction.
Designing the parameters of the reward function to optimize system performance is another compelling direction, as is learning optimal utility functions and/or signal dynamics from data.
Finally, identifying a biologically feasible description for an ecological system, such as hive site selection in bees, is an intriguing research topic.

\bibliography{bib/mendeley,bib/vsgc,bib/ieeexplore}

\begin{thebibliography}{16}
\providecommand{\natexlab}[1]{#1}

\bibitem[{Beaver and Malikopoulos(2021)}]{beaver2021overview}
Beaver, L.~E.; and Malikopoulos, A.~A. 2021.
\newblock An overview on optimal flocking.
\newblock \emph{Annual Reviews in Control}, 51: 88--99.

\bibitem[{Chakraa et~al.(2023)Chakraa, Gu{\'e}rin, Leclercq, and Lefebvre}]{chakraa2023optimization}
Chakraa, H.; Gu{\'e}rin, F.; Leclercq, E.; and Lefebvre, D. 2023.
\newblock Optimization techniques for Multi-Robot Task Allocation problems: Review on the state-of-the-art.
\newblock \emph{Robotics and Autonomous Systems}, 104492.

\bibitem[{Chremos, Beaver, and Malikopoulos(2020)}]{chremos2020game}
Chremos, I.~V.; Beaver, L.~E.; and Malikopoulos, A.~A. 2020.
\newblock A game-theoretic analysis of the social impact of connected and automated vehicles.
\newblock In \emph{2020 IEEE 23rd international conference on intelligent transportation systems (ITSC)}, 1--6. IEEE.

\bibitem[{Dunbabin and Marques(2012)}]{dunbabin2012robots}
Dunbabin, M.; and Marques, L. 2012.
\newblock Robots for environmental monitoring: Significant advancements and applications.
\newblock \emph{IEEE Robotics \& Automation Magazine}, 19(1): 24--39.

\bibitem[{Egerstedt et~al.(2018)Egerstedt, Pauli, Notomista, and Hutchinson}]{Egerstedt2018RobotAutonomy}
Egerstedt, M.; Pauli, J.~N.; Notomista, G.; and Hutchinson, S. 2018.
\newblock {Robot ecology: Constraint-based control design for long duration autonomy}.
\newblock \emph{Annual Reviews in Control}, 46: 1--7.

\bibitem[{Kanakia, Touri, and Correll(2016)}]{Kanakia2016ModelingGame}
Kanakia, A.; Touri, B.; and Correll, N. 2016.
\newblock {Modeling multi-robot task allocation with limited information as global game}.
\newblock \emph{Swarm Intelligence}, 10(2): 147--160.

\bibitem[{Khamis, Hussein, and Elmogy(2015)}]{khamis2015multi}
Khamis, A.; Hussein, A.; and Elmogy, A. 2015.
\newblock Multi-robot task allocation: A review of the state-of-the-art.
\newblock \emph{Cooperative robots and sensor networks 2015}, 31--51.

\bibitem[{Krieger and Billeter(2000)}]{krieger2000call}
Krieger, M.~J.; and Billeter, J.-B. 2000.
\newblock The call of duty: Self-organised task allocation in a population of up to twelve mobile robots.
\newblock \emph{Robotics and Autonomous Systems}, 30(1-2): 65--84.

\bibitem[{Krieger, Billeter, and Keller(2000)}]{Krieger2000}
Krieger, M.~J.; Billeter, J.-B.; and Keller, L. 2000.
\newblock {Ant-like task allocation and recruitment in cooperative robots}.
\newblock \emph{Nature}, 406: 992--995.

\bibitem[{Mahdavifar et~al.(2018)Mahdavifar, Beirami, Touri, and Shamma}]{Mahdavifar2018GlobalSharing}
Mahdavifar, H.; Beirami, A.; Touri, B.; and Shamma, J.~S. 2018.
\newblock Global Games With Noisy Information Sharing.
\newblock \emph{IEEE Transactions on Signal and Information Processing over Networks}, 4(3): 497--509.

\bibitem[{Martin et~al.(2023)Martin, Muros, Maestre, and Camacho}]{MARTIN2023104314}
Martin, J.~G.; Muros, F.~J.; Maestre, J.~M.; and Camacho, E.~F. 2023.
\newblock Multi-robot task allocation clustering based on game theory.
\newblock \emph{Robotics and Autonomous Systems}, 161: 104314.

\bibitem[{Morris and Shin(2001)}]{Morris2001GlobalApplications}
Morris, S.; and Shin, H.~S. 2001.
\newblock {Global Games: Theory and Applications}.
\newblock Technical report, Yale University.

\bibitem[{Naderi et~al.(2022)Naderi, Bundy, Whitney, Abedi, Weiskittel, and Contosta}]{naderi2022sharing}
Naderi, S.; Bundy, K.; Whitney, T.; Abedi, A.; Weiskittel, A.; and Contosta, A. 2022.
\newblock Sharing wireless spectrum in the forest ecosystems using artificial intelligence and machine learning.
\newblock \emph{International Journal of Wireless Information Networks}, 29(3): 257--268.

\bibitem[{Notomista and Egerstedt(2019)}]{Notomista2019Constraint-DrivenSystems}
Notomista, G.; and Egerstedt, M. 2019.
\newblock {Constraint-Driven Coordinated Control of Multi-Robot Systems}.
\newblock In \emph{Proceedings of the 2019 American Control Conference}.

\bibitem[{Notomista et~al.(2019)Notomista, Mayya, Hutchinson, and Egerstedt}]{notomista2019optimal}
Notomista, G.; Mayya, S.; Hutchinson, S.; and Egerstedt, M. 2019.
\newblock An optimal task allocation strategy for heterogeneous multi-robot systems.
\newblock In \emph{2019 18th European control conference (ECC)}, 2071--2076. IEEE.

\bibitem[{Notomista, Pacchierotti, and Giordano(2022)}]{notomista2022multi}
Notomista, G.; Pacchierotti, C.; and Giordano, P.~R. 2022.
\newblock Multi-robot persistent environmental monitoring based on constraint-driven execution of learned robot tasks.
\newblock In \emph{2022 International Conference on Robotics and Automation (ICRA)}, 6853--6859. IEEE.

\end{thebibliography}

\end{document}